\documentclass[final]{article}

\usepackage{aaai}
\usepackage{times}

\usepackage{subfig}
\usepackage{amsthm}
\usepackage{amssymb}
\usepackage{amsmath}
\usepackage{mathtools}
\usepackage{mathrsfs}
\usepackage{xspace} 
\usepackage[shortlabels]{enumitem}
\usepackage{url}
\usepackage{booktabs}
\usepackage{graphicx}
\usepackage{algorithm}
\usepackage{algorithmicx}
\usepackage[noend]{algpseudocode}
\usepackage{tikz}
\usepackage{multirow}
\usepackage{nicefrac}
\usepackage{macros}
\usepackage{tikz}
\usetikzlibrary{arrows,calc}

\title{Dynamic Bayesian Ontology Languages}

\author{
\.Ismail \.Ilkan Ceylan\\
	    Theoretical Computer Science \\ 
	    TU Dresden, Germany \\
		\url{ceylan@tcs.inf.tu-dresden.de}   
	\And 
		Rafael Pe\~naloza\\
		KRDB Research Centre \\ 
		Free University of Bozen-Bolzano, Italy \\
		\url{rafael.penaloza@unibz.it}   
}

\newcommand{\ax}[1]{\ensuremath{\left<#1\right>}\xspace}
\newcommand{\bl}{\ensuremath{\mathcal{B{\kern-.02em}L}}\xspace}
\newcommand{\dbl}{\ensuremath{\mathcal{DB{\kern-.02em}L}}\xspace}
\newcommand{\bel}{\ensuremath{\mathcal{BE{\kern-.02em}L}}\xspace}

\newcommand{\cn}[1]{\scalebox{.89}[1]{\ensuremath{\mathtt{#1}}}}

\theoremstyle{plain}
\newtheorem{theorem}{Theorem}
\newtheorem{lemma}[theorem]{Lemma}
\newtheorem{proposition}[theorem]{Proposition}

\theoremstyle{definition}
\newtheorem{definition}[theorem]{Definition}
\newtheorem{example}[theorem]{Example}

\newcommand{\Ibb}{\ensuremath{\mathbb{I}}\xspace}

\DeclareMathOperator*{\no}{\sim}

\begin{document}

\maketitle

\begin{abstract}
Many formalisms combining ontology languages with uncertainty, usually in the form
of probabilities, have been studied over the years. Most of these formalisms, however,
assume that the probabilistic structure of the knowledge remains static over time.
We present a general approach for extending ontology languages to handle time-evolving
uncertainty represented by a dynamic Bayesian network. We show how reasoning in 
the original language and dynamic Bayesian inferences can be exploited for effective
reasoning in our framework. 
\end{abstract}

\section{Introduction}
Description Logics (DLs)~\cite{BCM+-07} are a well-known family of knowledge representation 
formalisms that have been successfully employed for encoding the knowledge of many application
domains. In DLs, knowledge is represented through a finite set of axioms, usually called
an \emph{ontology} or knowledge base. In essence, these axioms are atomic pieces of knowledge 
that provide explicit information of the domain. When mixed together in an ontology, these axioms 
may imply some additional knowledge that is not explicitly encoded. Reasoning is the act of 
making this implicit knowledge explicit through an entailment relation.

Some of the largest and best-maintained DL ontologies represent knowledge from the bio-medical
domains. For instance, the NCBO Bioportal%
\footnote{\url{http://bioportal.bioontology.org/}}
contains 420 ontologies of various sizes.
In the bio-medical fields it is very common to have only uncertain knowledge. The certainty that 
an expert has on an atomic piece of knowledge may have arisen from a statistical test,
or from possibly imprecise measurements, for example.
It thus becomes relevant to extend DLs to represent and reason with uncertainty.

The need for probabilistic extensions of DLs has been observed for over two decades already.
To cover it, many different formalisms have been introduced~\cite{Jaeg-KR'94,LuSt'08,LuSc-10,KlPa-URSW'11}.
The differences in these logics range from the underlying classical DL used, to the semantics, to
the assumptions made on the probabilistic component.
One of the main issues that these logics need to handle is the representation of joint probabilities, in
particular when the different axioms are not required to be probabilistically independent.
A recent approach solves this issue by dividing the ontology into \emph{contexts}, which intuitively
represent axioms that must appear together. The probabilistic knowledge is expressed through a 
Bayesian network that encodes the joint probability distribution of these contexts.
Although originally developed as an extension of the DL \el~\cite{CePe-IJCAR'14}, 
the framework has been extended to arbitrary
ontology languages with a monotone entailment relation~\cite{CePe-DL'14}.

One common feature of the probabilistic extensions of DLs existing in the literature is that they
consider the probability distribution to be static. For many applications, this assumption does not hold:
the probability of a person to have gray hair increases as time passes, as does the probability of a 
computer component to fail.
To the best of our knowledge, there is so far no extension of DLs that can handle evolving probabilities
effectively.

In this paper, we describe a general approach for extending ontology languages to handle evolving
probabilities. By extension, our method covers all DLs, but is not limited to them. 
The main idea is to adapt the formalism from~\cite{CePe-DL'14} to use \emph{dynamic} Bayesian
networks (DBNs)~\cite{Mu'02} as the underlying uncertainty structure 
to compactly encode evolving probability distributions.

Given an arbitrary ontology language \Lmc, we define its dynamic Bayesian extension \dbl. We show that
reasoning in \dbl can be seamlessly divided into the probabilistic computation over the DBN, and the logical
component with its underlying entailment relation. In order to reduce the number of entailment checks, 
we compile a so-called context formula, which encodes all contexts in which a given consequence holds.

Related to our work are relational BNs~\cite{Ja-UAI'97} and their extensions. In contrast to relational BNs, 
we provide a tight coupling between the logical formalism and the DBN, which allows us to describe evolving 
probabilities while keeping the intuitive representations of each individual component.
Additionally, restricting the logical formalism to specific ontology languages provides an opportunity for finding 
effective reasoning algorithms.

%

\section{Bayesian Ontology Languages}

To remain as general as possible, we do not fix a specific logic, but consider an arbitrary
\emph{ontology language} \Lmc consisting of two infinite sets \Amf and \Cmf of \emph{axioms} and 
\emph{consequences}, respectively, and a class $\Omf\subseteq\wp_{\sf fin}(\Amf)$ of finite sets
of axioms, called \emph{ontologies}, such that if $\Omc\in\Omf$, then $\Omc'\in\Omf$ for all 
$\Omc'\subseteq\Omc$.
The language \Lmc is associated to a class \Imf of \emph{interpretations} and 
an \emph{entailment relation}
${\models}\subseteq\Imf\times(\Amf\cup\Cmf)$.
An interpretation $\Imc\in\Imf$ is a \emph{model} of the ontology \Omc ($\Imc\models\Omc$) if
$\Imc\models\alpha$ for all $\alpha\in\Omc$.
\Omc \emph{entails} $c\in\Cmf$ ($\Omc\models c$) if every model
of \Omc entails $c$.
Notice that the entailment relation is monotonic; i.e., if $\Omc\models c$ and 
$\Omc\subseteq\Omc'\in\Omf$, then ${\Omc'\models c}$.
Any standard description logic (DL)~\cite{BCM+-07} is an ontology language of this kind; consequences
in these languages are e.g.\ concept unsatisfiability, concept subsumption, or query entailment.
However, many other ontology languages of varying expressivity and complexity properties exist.
For the rest of this paper, \Lmc is an arbitrary but fixed ontology language, with axioms \Amf, 
ontologies \Omf, consequences \Cmf, and interpretations \Imf.

As an example language we use the DL \el~\cite{BaBrLu-IJCAI'05}, which we briefly introduce here. 
Given two disjoint sets \NC and \NR, \el \emph{concepts} are built by the grammar rule
$C::=A\mid \top\mid C\sqcap C\mid \exists r.C$ where $A\in\NC$ and $r\in\NR$. \el \emph{axioms} 
and \emph{consequences} are expressions of the form $C\sqsubseteq D$, where $C$ and $D$ are concepts.
An \emph{interpretation} is a pair $(\Delta^\Imc,\cdot^\Imc)$ where $\Delta^\Imc$ is a non-empty set
and $\cdot^\Imc$ maps every $A\in\NC$ to $A^\Imc\subseteq\Delta^\Imc$ and every $r\in\NR$ to
$r^\Imc\subseteq\Delta^\Imc\times\Delta^\Imc$. This function is extended to concepts
by $\top^\Imc:=\Delta^\Imc$, ${(C\sqcap D)^\Imc:=C^\Imc\cap D^\Imc}$, and 
${(\exists r.C)^\Imc:=\{d\mid\exists e\in C^\Imc.(d,e)\in r^\Imc\}}$.
This in\-ter\-pre\-ta\-tion \emph{entails} the axiom (or consequence) $C\sqsubseteq D$ iff
$C^\Imc\subseteq D^\Imc$.

%
The Bayesian ontology language \bl extends \Lmc by associating each axiom in an ontology with a 
context, which intuitively describes the situation in which the axiom is required to hold. 
The knowledge of which context applies is uncertain, and expressed through a Bayesian 
network~\cite{CePe-DL'14}.

Briefly, a \emph{Bayesian network}~\cite{Darw-09} is a pair 
$\Bmc=(G,\Phi)$, where $G=(V,E)$ is a finite directed acyclic graph~(DAG) whose nodes represent 
Boolean random variables, 
and $\Phi$ contains, for every $x\in V$, a conditional probability distribution ${P_\Bmc(x\mid\pi(x))}$ of $x$ given 
its parents $\pi(x)$. 
Every variable $x\in V$ is conditionally independent of its non-descendants given 
its parents. Thus, the BN \Bmc defines a unique joint probability distribution (JPD) over $V$:
\[
P_{\Bmc}(V)=\prod_{x\in V}P_\Bmc(x\mid\pi(x)).
\]

Let $V$ be a finite set of variables.
A  \emph{$V$-context} is a consistent set of literals over $V$.
A \emph{$V$-axiom} is an expression of the form $\ax{\alpha:\kappa}$ where  $\alpha\in\Amf$ is an axiom and 
$\kappa$ is a $V$\mbox{-}context.
A \emph{$V$-ontology} is a finite set \Osf of $V$-axioms, such that 
$\{\alpha\mid\ax{\alpha:\kappa}\in\Osf\}\in\Omf$.
%
%
A \emph{\bl knowledge base} (KB) over $V$ is a pair ${\Kmc=(\Osf,\Bmc)}$ where $\Bmc$ is a BN over $V$ and
\Osf is a $V$-ontology.%

We briefly illustrate these notions over the language \bel, an extension of the DL \el, in the following Example.

\begin{example}
\label{exa:run}
Consider the \bel KB $\Kmc_1=(\Bmc_1,\Osf_1)$ where
\begin{align*}
\Osf_1 {=} \{ & \ax{\cn{Comp}\sqsubseteq\exists\cn{use}.\cn{Mem}\sqcap\exists\cn{use}.\cn{CPU}:\emptyset}, \\
			& \ax{\exists\cn{use}.\cn{FailMem}\sqsubseteq\cn{FailComp}:\{x\}}, \\
			& \ax{\exists\cn{use}.\cn{FailCPU}\sqsubseteq\cn{FailComp}:\{x\}}, \\
			& \ax{\exists\cn{use}.\cn{FailMem}\sqcap\exists\cn{use}.\cn{FailCPU}\sqsubseteq\cn{FailComp}{:}\{\neg x\}}, \\
			& \ax{\cn{Mem}\sqsubseteq\cn{FailMem}:\{y\}}, 
			  \ax{\cn{CPU}\sqsubseteq\cn{FailCPU}:\{z\}}
			\},
\end{align*}
and $\Bmc_1$ is the BN shown in Figure~\ref{fig:bn}. 
\begin{figure}[t]
\centering
\includegraphics[width=\columnwidth]{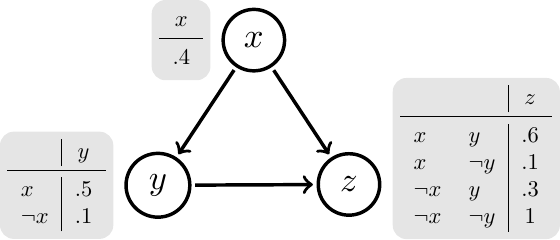}
\caption{The BN $\Bmc_1$ over the variables $V=\{x,y,z\}$}
\label{fig:bn}
\end{figure}
\end{example}
This KB represents a computer failure scenario, where $x$ stands for a critical situation, $y$ represents the 
memory failing, and $z$ the CPU failing.

The \emph{contextual semantics} is defined by extending interpretations to evaluate also the variables from $V$.
A \emph{$V$\mbox{-}interpretation} is a pair $\Isf=(\Imc,\Vmc^\Isf)$ where ${\Imc\in\Imf}$ and 
$\Vmc^\Isf$ is a propositional interpretation over the variables $V$.
The $V$-interpretation $\Isf=(\Imc,\Vmc^\Isf)$ is a \emph{model} of \ax{\alpha:\kappa} 
($\Isf\models \ax{\alpha:\kappa}$), where $\alpha\in\Amf$, iff 
$(\Vmc^\Isf \not\models_p \kappa)$ {\footnote{We use $\models_p$ to distinguish propositional entailment from $\models$.}} 
 or $(\Imc\models\alpha)$.%

It is a \emph{model} of the $V$-ontology \Osf iff it is a model of all the $V$-axioms in \Osf.
It \emph{entails} $c\in\Cmf$ if $\Imc\models c$.
The intuition behind this semantics is that an axiom is evaluated to true by all models provided 
it is in the right context.

Given a $V$-ontology \Osf, every propositional interpretation, or \emph{world}, \Wmc on $V$  defines an ontology 
%
$
{\Osf_\Wmc := \{\alpha\mid \ax{\alpha:\kappa}\in \Osf, \Wmc \models_p \kappa\}.}
$
Consider the KB $\Kmc_1$ provided in Example~\ref{exa:run}: The world ${\Wmc=\{x,\neg y, z\}}$ defines the ontology
\begin{align*}
\Osf_\Wmc {=} \{ &  \cn{Comp}\sqsubseteq\exists\cn{use}.\cn{Mem}\sqcap\exists\cn{use}.\cn{CPU}, \\
			& \exists\cn{use}.\cn{FailMem}\sqsubseteq\cn{FailComp}, \\
			& \exists\cn{use}.\cn{FailCPU}\sqsubseteq\cn{FailComp}, 
			 \cn{CPU}\sqsubseteq\cn{FailCPU}
			\}.
\end{align*}

Intuitively, a 
contextual ontology
is a compact representation of exponentially many ontologies from \Lmc; one for 
each
world \Wmc.
The uncertainty in \bl is expressed by the BN \Bmc, which is interpreted using multiple world semantics.

\begin{definition}[probabilistic interpretation]
A \emph{probabilistic interpretation} is a pair $\Pmc=(\Ibb,P_\Ibb)$, where \Ibb is a set of 
$V$\mbox{-}interpretations and $P_\Ibb$ is a probability distribution over \Ibb such that $P_\Ibb(\Isf)>0$ only 
for finitely many interpretations $\Isf\in\Ibb$.
It is a \emph{model} of the $V$-ontology \Osf if every $\Isf\in\Ibb$ is a model of \Osf.
\Pmc is \emph{consistent} with the BN \Bmc if for every valuation \Wmc of the variables in $V$ it holds that 
\[
\sum_{\mathclap{\Isf\in\Ibb,\, \Vmc^\Isf=\Wmc}}\; P_\Ibb(\Isf)=P_\Bmc(\Wmc).
\]
The probabilistic interpretation \Pmc is a \emph{model} of the KB $(\Bmc,\Osf)$ iff it is a model of 
\Osf and consistent with \Bmc.
\end{definition}
The fundamental reasoning task in \bl, probabilistic entailment, consists in finding the probability of observing
a consequence $c$; that is, the probability of being at a context where $c$ holds.

\begin{definition}[probabilistic entailment]
Let $c\in\Cmf$, and \Kmc a \bl KB. 
%
The \emph{probability of $c$} w.r.t.\ the probabilistic interpretation 
$\Pmc=(\Ibb,P_\Ibb)$ is
$P_\Pmc(c):=\sum_{(\Imc,\Wmc)\in\Ibb,\Imc\models c}P_\Ibb(\Imc,\Wmc).$
The \emph{probability of $c$ w.r.t.\ \Kmc} is 
${P_\Kmc(c):= \inf_{\Pmc\models\Kmc}P_\Pmc(c).}$
\end{definition}

It has been shown that to compute the conditional probability of a consequence $c$, 
it suffices to test, for each world \Wmc, whether $\Osf_\Wmc$ entails $c$~\cite{CePe-IJCAR'14}.
%
\begin{proposition}
\label{prop:basic}
Let $\Kmc=(\Bmc,\Osf)$ be a \bl KB and $c\in\Cmf$. Then
$P_\Kmc(c)=\sum_{{\Osf_\Wmc\models c}}P_\Bmc(\Wmc)$.
\end{proposition}
This means that reasoning in \bl can be reduced to exponentially many entailment tests in the classical language 
\Lmc. 
For some logics, this exponential enumeration of worlds can be avoided~\cite{CePe-JELIA'14}. However,
this depends on the properties of the ontological language and its entailment relation, and cannot be guaranteed
for arbitrary languages.

Another relevant problem is to compute the probability of a consequence given some partial information about
the context. Given a context $\kappa$, the conditional probability ${P_\Kmc(c\mid\kappa)}$ is defined via the
rule ${P_\Kmc(c,\kappa)=P_\Kmc(c\mid\kappa)P_\Bmc(\kappa)}$, where 
\[{P_\Kmc(c,\kappa)=\sum_{\Osf_\Wmc\models c,\Wmc\models_p\kappa}P_\Bmc(\Wmc).}\]
%
For simplicity, in the rest of this paper we consider only unconditional consequences. However, it should be 
noted that all results can be transferred to the conditional case.
%
%
%
%

\begin{example}
Consider again the KB $\Kmc_1=(\Bmc_1,\Osf_1)$ from Example~\ref{exa:run} and the consequence $\cn{Comp}\sqsubseteq\cn{FailComp}$.
We are interested in finding the probability of the computer to fail, \ie $P_{\Kmc_1}(\cn{Comp}\sqsubseteq\cn{FailComp})$. This can be computed by enumerating all worlds $\Wmc$ for which $\Osf_\Wmc \models \cn{Comp}\sqsubseteq\cn{FailComp}$, which yields the probability 0.238.
\end{example}

As seen, it is possible to extend any ontological language to allow for probabilistic reasoning based 
on a Bayesian network. We now further extend this formalism to be able to cope with controlled updates
of the probabilities over time.  

\section{Dynamic Bayesian Ontology Languages}

With \bl, one is able to represent and reason about the uncertainty of the current context, and the consequences
that follow from it. In that setting, the joint probability distribution of the contexts, expressed by the BN,
is known and fixed. In some applications, see especially~\cite{SaKa'10} this probability distribution may change over time.
For example, as the components of a computer age, their probability of failing
increases. The new probability depends on how likely it was for the component to fail previously, and 
the ageing factors to which it is exposed. We now extend \bl to handle these cases, by considering
\emph{dynamic} BNs as the underlying formalism for managing uncertainty over contexts.

Dynamic BNs (DBNs)~\cite{DeKa-CI'89,Mu'02} extend BNs to provide a compact representation of 
evolving joint probability distributions for a fixed set of random variables. The update of the JPD is 
expressed through a two-slice BN, which expresses the probabilities at the next point in time, given
the current context. 

\begin{definition}[DBN]
Let $V$ be a finite set of Boolean random variables. A \emph{two-slice BN} (TBN) over $V$ is a pair
$(G,\Phi)$, where $G=(V\cup V',E)$ is a DAG containing no edges
 between elements of $V$, 
$V'=\{x'\mid x\in V\}$, and $\Phi$ contains, for every
$x'\in V'$ a conditional probability distribution $P(x'\mid\pi(x'))$ of $x'$ given its parents $\pi(x')$.
A \emph{dynamic Bayesian network} (DBN) over $V$  is a pair $\Dmc=(\Bmc_1,\Bmc_{\rightarrow})$ where 
$\Bmc_1$ is a BN over $V$, and $\Bmc_{\rightarrow}$ is a TBN over $V$.
\end{definition}

A TBN over $V=\{x,y,z\}$ is depicted in Figure~\ref{fig:tbn}.
\begin{figure}[t]
\centering
\includegraphics[width=\columnwidth]{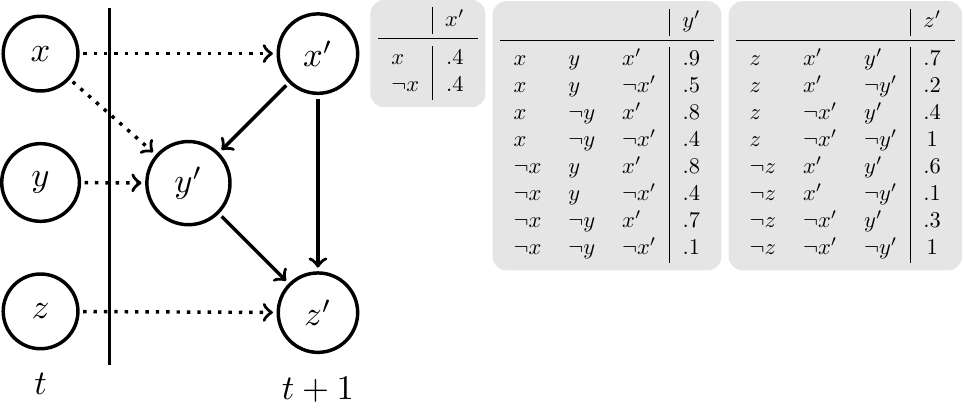}
\caption{ The TBN $\Bmc_\rightarrow$ over the variables $V=\{x,y,z\}$}
\label{fig:tbn}
\end{figure}
The set of nodes of the graph can be 
thought of as containing two disjoint copies of the random variables in $V$.
Then, the probability distribution at time $t+1$ depends on the distribution at time $t$. 
In the following we will use $V_t$ and $x_t$ for $x\in V$, to denote the variables in $V$ at time $t$.

As standard in BNs, the graph structure of a TBN encodes the conditional dependencies among the nodes: 
every node is independent of all its non-descendants given its parents. 
Thus, for a TBN \Bmc, 
the conditional probability distribution at time $t+1$ given time $t$ is
\[
P_{\Bmc}(V_{t+1}\mid V_{t})= \prod_{x'\in V'} P_{\Bmc}(x'\mid \pi(x')).
\]
We further assume the Markov property: the probability of the future state is independent from the
past, given the present state.

In addition to the TBN, a DBN contains a BN $\Bmc_1$ that encodes the JPD of $V$ at the beginning of the
evolution. Thus, the DBN $\Dmc=(\Bmc_1,\Bmc_\rightarrow)$ defines, for every $t\ge 1$, the unique 
probability distribution
\[
P_{\Bmc}(V_{t})= P_{\Bmc_1}(V_1) \prod_{i=2}^t\prod_{x\in V} P_{\Bmc_\rightarrow}(x_i\mid\pi(x_i)).
\]
%
Intuitively, the distribution at time $t$ is defined by unraveling the DBN starting from $\Bmc_1$, using the
two-slice structure of $\Bmc_\rightarrow$ until $t$ copies of $V$ have been created. This produces 
a new BN $\Bmc_{1:t}$ encoding the distribution over time of the different variables. Figure~\ref{fig:unrav}
depicts the unraveling to $t=3$ of the DBN $(\Bmc_1,\Bmc_\rightarrow)$ where
$\Bmc_1$ and $\Bmc_\rightarrow$ are the networks depicted in Figures~\ref{fig:bn} and~\ref{fig:tbn}, 
respectively.
\begin{figure}[t]
\centering
\includegraphics[width=\columnwidth]{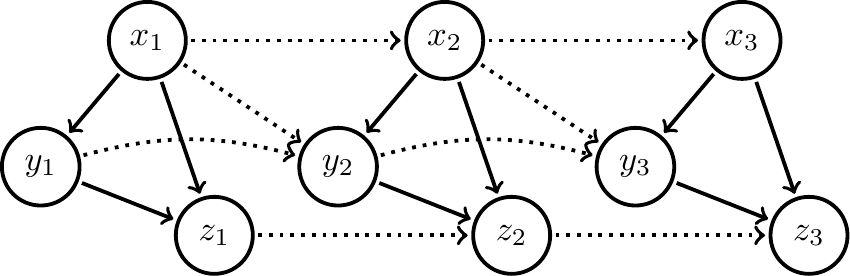}
\caption{Three step unraveling $\Bmc_{1:3}$ of $(\Bmc_1,\Bmc_\rightarrow)$}
\label{fig:unrav}
\end{figure}
The conditional probability tables of each node given its parents (not depicted) are those of $\Bmc_1$
for the nodes in $V_1$, and of $\Bmc_\rightarrow$ for nodes in $V_2\cup V_3$.
Notice that $\Bmc_{1:t}$ has $t$ copies of each random variable in $V$. 
For a given $t\ge 1$, we call $\Bmc_t$ the BN obtained from the unraveling $\Bmc_{1:t}$ 
of the DBN to time $t$, and eliminating all variables not in $V_t$. In particular, we have that
$P_{\Bmc_t}(V)=P_{\Bmc_{1:t}}(V_t)$.

The dynamic Bayesian ontology language \dbl is very similar to \bl, except that the probability distribution
of the contexts evolves accordingly to a DBN.

\begin{definition}[\dbl KB]
A \dbl knowledge base (KB) is a pair ${\Kmc=(\Dmc,\Osf)}$ where 
${\Dmc=(\Bmc_1,\Bmc_{\rightarrow})}$ is a DBN over $V$ and \Osf is a $V$-ontology.
Let ${\Kmc=(\Dmc,\Osf)}$ be a \dbl KB  over $V$.
A \emph{timed probabilistic interpretation} is an infinite sequence 
$\Pmf=(\Pmc_t)_{t\ge 1}$ of probabilistic interpretations.
\Pmf is a \emph{model} of $\Kmc$ if for every $t$, 
$\Pmc_t$ is a model of the \bl KB $(\Bmc_t,\Osf)$.
\end{definition}
In a nutshell, a DBN can be thought of as a compact representation of an infinite sequence of
BNs $\Bmc_1,\Bmc_2,\ldots$ over $V$. 
Following this idea, a \dbl KB expresses an infinite sequence of \bl KBs, where
the ontological component remains unchanged, and only the probability distribution of the contexts
evolves over time.
A timed probabilistic interpretation \Pmf simply interprets each of these \bl KBs, at the corresponding point in
time. To be a model of a \dbl KB, \Pmf must then be a model of all the associated \bl KBs.

Before describing the reasoning tasks for \dbl and methods for solving them, we show how the 
computation of all the contexts that entail a consequence can be reduced to the enumeration of the 
worlds satisfying a propositional formula.

\section{Compiling Contextual Knowledge}

From Proposition~\ref{prop:basic}, we see that reasoning in \bl can be reduced to checking, for every
world \Wmc, whether $\Osf_\Wmc\models c$. This reduces probabilistic reasoning to a sequence of standard
entailment tests over the original language \Lmc.
However, each of these entailments might be very expensive. For example,
in the very expressive DL $\mathcal{SHOIQ}$, deciding an entailment is already
\NExpTime-hard~\cite{To'00}.
Rather than repeating this reasoning step for every world, it makes sense to try
to identify the relevant worlds \emph{a priori}.
We do this through the computation of a \emph{context formula}.

\begin{definition}[context formula]
Let \Osf be a $V$-ontology, and $c\in\Cmf$. A \emph{context formula} for $c$ w.r.t.\ \Osf is a propositional
formula $\phi$ such that for every interpretation \Wmc of the variables in $V$, it holds that
$\Osf_\Wmc \models c$ iff $\Wmc \models_p \phi$.
\end{definition}

The idea behind this formula is that, for finding whether $\Osf_\Wmc\models c$, it suffices to check
whether the valuation \Wmc satisfies $\phi$. This test requires only linear time on the length of 
the context formula.
The context formula can be seen as a generalization of the pinpointing formula~\cite{BaPe-JLC'10}
and the boundary~\cite{BaKP-JWS'12}, defined originally for classical ontology languages.

\begin{example}
Consider again the $V$-ontology $\Osf_1$ from Example~\ref{exa:run}. 
The formula $\phi_1:=(x\land (y\lor z))\lor(\neg x\land y\land z)$ is a context formula for 
$\cn{Comp}\sqsubseteq\cn{FailComp}$ w.r.t.\ $\Osf_1$.
In fact, the valuation $\{x,\neg y, z\}$ satisfies this formula.
\end{example}

Clearly, computing the context formula must be at least as hard as deciding an entailment in \Lmc: if
we label every axiom in a classical \Lmc-ontology \Omc with the same propositional variable $x$, then
the boundary formula of $c$ w.r.t.\ this $\{x\}$\mbox{-}ontology is $x$ iff $\Omc\models c$.
On the other hand, the algorithm used for deciding the entailment relation can usually be modified
to compute the context formula. Using arguments similar to those developed for axiom 
pinpointing~\cite{BaPe-JAR'09,KPHS-ISWC'07}, it can be shown that for most description logics,
computing the context formula is not harder, in terms of computational complexity, than standard
reasoning. 
In particular this holds for any arbitrary ontology language whose entailment relation is \ExpTime-hard.
This formula can also be compiled into a more efficient data structure like binary decision 
diagrams~\cite{Lee'59}.
Intuitively, this means that we can compute this formula using the same amount of resources needed for
only one entailment test, and then use it for verifying whether the sub-ontology defined by a world
entails the consequence in an efficient way.

\section{Reasoning in \dbl}

Rather than merely computing the probability of currently observing a consequence,
we are interested in computing the probability of a consequence to follow after some fixed number of time steps $t$.
\begin{definition}[probabilistic entailment with time]
\label{def:probent}
Let ${\Kmc=(\Dmc,\Osf)}$ be a \dbl KB and $c\in\Cmf$. 
Given a timed interpretation \Pmf and $t\ge 1$, 
the \emph{probability of $c$  at time $t$} \wrt \Pmf is
${P_{\Pmf}(c[t]):=P_{\Pmc_t}(c)}$.
The \emph{probability of $c$ at time $t$ w.r.t.\ \Kmc} is 
$P_\Kmc(c[t]):=\inf_{\Pmf\models\Kmc}P_{\Pmf}(c[t]).$
\end{definition}
%


We show that probabilistic entailment over a fixed time bound can be reduced to probabilistic entailment defined for BOLs by unravelling the DBN.
\begin{lemma}
\label{lem:red}
Let $\Kmc=(\Dmc,\Osf)$ be a \dbl KB, $c\in\Cmf$, and $t\geq 1$. Then the probability of $c$ at time $t$ \wrt \Kmc is given by
\[P_\Kmc(c[t])=\sum_{{\Osf_\Wmc\models c}}P_{\Bmc_t}(\Wmc)\]
\end{lemma}
\begin{proof}(Sketch)
A timed model \Pmf of \Kmc is a sequence of probabilistic interpretations
$\Pmc_1,\Pmc_2,\ldots$, where each $\Pmc_i$ is a model of the \bl KB 
$\Kmc_i:=(\Bmc_i,\Osf)$. We use this fact to show that
\begin{align}
P_\Kmc(c[t]) &=\inf_{\Pmf\models\Kmc}P_{\Pmf}(c[t])= \inf_{\Pmf\models\Kmc} P_{\Pmc_t}(c) \label{lten:one}\\ 
&=\inf_{\Pmc_t\models\Kmc_t} P_{\Pmc_t}(c) = P_{\Kmc_t}(c)\label{lten:two}\\
&=  \sum_{{\Osf_\Wmc\models c}}P_{\Bmc_t}(\Wmc), \label{lten:three}
\end{align}
where \eqref{lten:one} follows from Definition~\ref{def:probent}, 
\eqref{lten:two} holds by definition, 
and \eqref{lten:three} follows from Proposition~\ref{prop:basic}.
\end{proof}
Lemma~\ref{lem:red} provides a method for computing the probability of an entailment at a fixed time $t$. 
One can first generate the BN $\Bmc_t$, and then compute the probability w.r.t.\ $\Bmc_t$ of all the 
worlds that entail $c$. 
Moreover, using a context formula we can compile away the ontology and reduce reasoning to standard inferences in BNs, only.

\begin{theorem}
Let $\Kmc=(\Dmc,\Osf)$ be a \dbl KB, $c\in\Cmf$, $\phi$ a context formula for $c$ \wrt $\Osf$, 
and $t\ge 1$. Then the probability of $c$ at time $t$ \wrt \Kmc is given by ${P_\Kmc(c[t])=P_{\Bmc_t}(\phi)}$.
\end{theorem}
\begin{proof}
By Lemma~\ref{lem:red} and the definition of a context formula, we have
\[
P_\Kmc(c[t])=\sum_{{\Osf_\Wmc\models c}}P_{\Bmc_t}(\Wmc)
=\sum_{{\Wmc\models_p \phi}}P_{\Bmc_t}(\Wmc) = P_{\Bmc_t}(\phi),
\]
which proves the result.
\end{proof}
This means that one can first compute a context formula for $c$ and then do probabilistic inferences 
over the DBN to detect the probability of satisfying $\phi$ at time $t$. For this, we can exploit any existing DBN 
inference method. One option is to do variable elimination over the $t$-step unraveled DBN $\Bmc_{1:t}$,
to compute $\Bmc_t$.
Assuming that $t$ is fixed, it suffices to make $2^{|V|}$ inferences (one for each world) over $\Bmc_t$ and
the same number of propositional entailment tests over the context formula.
If entailment in \Lmc is already exponential, then computing the probability of $c$ at time $t$ is as hard
as deciding entailments.

The previous argument only works assuming a fixed time point $t$. 
Since it depends heavily on computing $\Bmc_t$ (e.g., via variable elimination), it does not scale
well as $t$ increases.
Other methods have been proposed for exploiting the recursive structure of the DBN.
For instance, one can use the algorithm described in  \cite{VlMeBr-StarAI'14} that provides linear scalability
over time.
The main idea is to compile the structure into an arithmetic circuit~\cite{Darw-09} and use
forward and backward message passing~\cite{Mu'02}.
%

While computing the probability of a consequence at a fixed point in time $t$ is a relevant task,
it is usually more important to know whether the consequence can be observed within a given time 
limit. In our computer example, we would be interested in finding the probability of the system failing 
within, say, the following twenty steps.

Abusing of the notation, we use the expression $\no c$, $c\in\Cmf$, to denote that the consequence
$c$ does not hold; i.e., $\Imc\models\no c$ iff $\Imc\not\models c$. Thus, for example,
$\Pmc_\Kmc(\no c[t])$ is the probability of $c$ not holding at time $t$.
To find the probability of observing $c$ in the first $t$ time steps, one can alternatively compute the probability
of not observing $c$ in any of those steps. Formally, for a timed interpretation \Pmf and $t\in\mathbb{N}$, we define
$$P_\Pmf(c[1:t]):=1-P_\Pmf(\no c[1],\ldots,\no c[t]).$$
%
%
\begin{definition}[time bounded probabilistic entailment]
The \emph{probability of observing $c$ in at most $t$ steps} w.r.t.\ the \dbl KB \Kmc is 
$P_\Kmc(c[1:t]):=\inf_{\Pmf\models\Kmc}P_\Pmf(c[1:t])$.
\end{definition}

Just as before, given a constant $t\ge 1$, it is possible to compute $P_\Kmc(c[1:t])$ by looking at
the $t$-step unraveling of \Dmc. More precisely, to compute $P_\Pmf(\no c[1],\ldots,\no c[t])$, it suffices
to look at all the valuations \Wmc of $\bigcup_{i=1}^t V_i$ such that for all $i,1\le i\le t$, it holds that
$\Osf_{\Wmc(i)}\models\no c$.
These valuations correspond to an evolution of the system where the consequence $c$ is not observed in 
the first $t$ steps. The probability of these valuations w.r.t.\ $\Bmc_{1:t}$ then yields the probability of
\emph{not} observing this consequence. We thus get the following result.

\begin{theorem}
\label{thm:timebound}
Let $\Kmc=(\Dmc,\Osf)$ be a \dbl KB, $c\in\Cmf$, $\phi$ a context formula for $c$ \wrt $\Osf$, and 
$t\ge 1$. Then
${P_\Kmc(c[1:t])=\sum_{\Wmc\mid \exists i.\Wmc(i)\models_p \phi}P_{\Bmc_{1:t}}(\Wmc)}.$
\end{theorem}
\begin{proof}[Proof (Sketch)]
Using the pithy interpretations of the crisp ontologies $\Omc_{\Wmc(i)}$, we can build a timed interpretation
$\Pmf_0$  such that $P_{\Pmf_0}(c[1:t])=\sum_{\Wmc\mid \exists i.\Wmc(i)\models_p \phi}P_{\Bmc_{1:t}}(\Wmc)$,
in a way similar to Theorem~12 of \cite{CePe-IJCAR'14}.
The existence of another timed interpretation $\Pmf$ such that $P_{\Pmf}(c[1:t])<P_{\Pmf_0}(c[1:t])$
contradicts the properties of the pithy interpretations. Thus, we obtain that
${P_{\Pmf_0}(c[1:t])=\inf_{\Pmf\models\Kmc}P_\Pmf(c[1:t])=P_\Kmc(c[1:t]).}$
\end{proof}

This means that the probability of observing a consequence within a fixed time-bound $t$ can be computed
by simply computing the context formula
and then performing probabilistic
\emph{a posteriori} computations over the unraveled BN. 
In our running example, the probability of observing a computer failure in the next 20 steps is simply
\[
{P_\Kmc(\cn{Comp}\sqsubseteq\cn{FailComp}[1:20])=
	\sum_{\Wmc\mid \exists i.\Wmc(i)\models_p \phi}P_{\Bmc_{i}}(\phi)}.
\]
Thus, the 
computational complexity of reasoning is not affected by introducing the dynamic evolution of the
BN, as long as the time bound is constant. Notice, however, that the number of possible valuations grows 
exponentially on the time bound $t$. Thus,
for large intervals, this approach becomes unfeasible.



By extending the time limit indefinitely, we can also find the probability of \emph{eventually} observing
the consequence $c$ (e.g., the probability of the system ever failing). 
The \emph{probability of eventually observing $c$} w.r.t.\ \Kmc is given by
${P_\Kmc(c[\infty]):=\lim_{t\to\infty}P_\Kmc(c[1:t])}$.
Notice that $P_\Kmc(c[1:t])$ is monotonic on $t$ and bounded by $1$; hence $P_\Kmc(c[\infty])$ is well defined.

Observe that Theorem~\ref{thm:timebound} cannot be used
to compute the probability of \emph{eventually} observing $c$ since one cannot necessarily predict the 
changes in probabilities of finding worlds that entail the consequence $c$. 
Rather than considering these increasingly large BNs separately, we can exploit methods developed  for probability distributions that evolve over time.
This will also allow us to extract more information from \dbl KBs.
%

It is easy to see that every TBN defines a time-homogeneous Markov chain over a finite state space.
More precisely, if \Bmc is a TBN over $V$, then $\Mmc_\Bmc$ is the Markov chain,
where every valuation \Wmc of the variables in $V$ is a state and the transition probability distribution
given the current state $\Wmc$ is described by the BN obtained from adding \Wmc as evidence to the first
slice of \Bmc. 
For example, the TBN
$\Bmc_\rightarrow$ from Figure~\ref{fig:tbn} yields the conditional probability
distribution given that $\{x,y,z\}$ was observed at time $t$ depicted in Figure~\ref{fig:Mcond}.
\begin{figure}[t]
\centering
\includegraphics[width=\columnwidth]{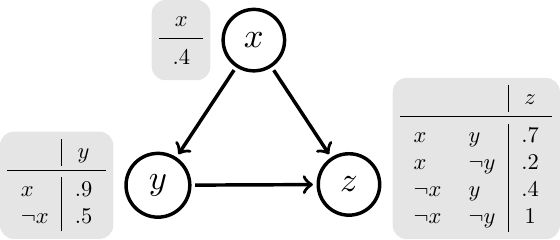}
\caption{$P(V\mid \{x,y,z\}_t)$}
\label{fig:Mcond}
\end{figure}
From this, we can derive the probability of observing $\{x,y,z\}$ at time $t+1$ given that it was observed at
time $t$, which is
$P(\{x,y,z\}_{t+1}\mid \{x,y,z\}_t)=0.252$.

We extend the notions from Markov chains to TBNs in the obvious way. In particular, the TBN \Bmc
is \emph{irreducible} if for every two worlds $\Vmc,\Wmc$, the probability of eventually reaching
$\Wmc$ given $\Vmc$ is greater than 0. It is \emph{aperiodic} if for every world \Wmc there is an
$n_\Wmc$ such that for all $n\ge n_\Wmc$, it holds that $P(\Wmc_n\mid \Wmc)>0$. 
A distribution $P_W$ over the worlds is \emph{stationary} if 
$\sum_{\Wmc}P(\Vmc\mid\Wmc)P_W(\Wmc)=P_W(\Vmc)$ holds for every world \Vmc.
It follows immediately that if \Bmc is irreducible and aperiodic, then it has a unique stationary
distribution~\cite{Harris1956}.

Given a TBN \Bmc over $V$, let now $\Delta_\Bmc$ be the set of all stationary distributions over the
worlds of $V$. 
For a world \Wmc, define $\delta_\Bmc(\Wmc):=\min_{P\in\Delta_\Bmc}P(\Wmc)$ to be the smallest 
probability assigned by any stationary distribution of \Bmc to \Wmc.
If $\delta_\Bmc(\Wmc)>0$, then we know that, regardless of the initial distribution, in the limit 
we will always be able to observe the world \Wmc with a constant positive probability.
In particular, this means that the probability of eventually observing \Wmc equals 1.
Notice moreover that this results is independent of the initial distribution used.

We can take this idea one step forward, and consider sets of worlds. For a propositional formula
$\phi$, let 
$$\delta_\Bmc(\phi):=\min_{P\in\Delta_\Bmc}\sum_{\Wmc\models_p \phi}P(\Wmc).$$
In other words, $\delta_\Bmc(\phi)$ expresses the minimum probability of satisfying $\phi$ in any 
stationary distribution of \Bmc. From the arguments above, we obtain the following theorem.

\begin{theorem}
Let $\Kmc=(\Dmc,\Osf)$ be a \dbl KB over $V$ with $\Dmc=(\Bmc_1,\Bmc_\rightarrow)$, $c\in\Cmf$, and
$\phi$ a context formula for $c$ w.r.t.\ \Osf. If $\delta_{\Bmc_\rightarrow}(\phi)>0$, then
$P_\Kmc(c[\infty])=1$.
\end{theorem}

In particular, if $\Bmc_\rightarrow$ is irreducible and aperiodic, $\Delta_\Bmc$ contains only one 
stationary distribution, which simplifies the computation of the function $\delta$. Unfortunately,
such a simple characterization of $P_\Kmc(c[\infty])$ cannot be given when $\delta_{\Bmc_\rightarrow}(\phi)=0$.
In fact, in this case the result may depend strongly on the initial distribution.

\begin{example}
Let $V=\{x\}$, $\Osf_2=\{\ax{A\sqsubseteq B:\{x\}}\}$, and 
consider the TBN $\Bmc'_\rightarrow$ over $V$ defined by $P(x'\mid x)=1$ and $P(x'\mid \neg x)=0$.
It is easy to see that any distribution over the valuations of $V$ is stationary. 
For every initial distribution $\Bmc$, if $\Kmc=(\Dmc,\Osf_2)$ where $\Dmc=(\Bmc,\Bmc'_\rightarrow)$,
then $P_\Kmc(A\sqsubseteq B[\infty])=P_\Bmc(x)$.
\end{example}

So far, our reasoning services have focused on predicting the outcome at future time steps, given the
current knowledge of the system. Based on our model of evolving probabilities, the distribution at any time
$t+1$ depends only on time $t$, if it is known. 
However, for many applications it makes sense to consider evidence that is observed throughout
several time steps. For instance, in our computer failure scenario, the DBN $\Bmc_\rightarrow$ 
ensures that, if at some point a critical situation is observed ($x$ is true), then the probability of observing 
a memory or CPU failure in the next step is higher. That is, the evolution of the probability distribution
is affected by the observed value of the variable $x$.

Suppose that we have observed over the first $t$ time steps that no critical situation has occurred, and
we want to know the probability of a computer failure. 
Formally, let $E$ be a consistent set of literals over $\bigcup_{i=1}^t V_i$. We want to compute 
the probability $P_\Kmc(c[t]\mid E)$ of observing $c$ at time $t$ given the evidence $E$.
This is just a special case of bounded probabilistic entailment, where the worlds are not only restricted \wrt the context formula but also \wrt the evidence $E$.


The efficiency of this approach depends strongly on the time bound $t$, but also on the structure of the
TBN $\Bmc_\rightarrow$. Recall that the complexity of reasoning in a BN depends on the tree-width of
its underlying DAG~\cite{PaML'98}. 
The unraveling of $\Bmc_\rightarrow$ produces a new DAG whose tree-width might increase with each
unraveling step, thus impacting the reasoning methods negatively.

\section{Conclusions}
We have introduced a general approach for extending ontology languages to handle 
time-evolving probabilities with the help of a DBN. 
Our framework can be instantiated to any language with a monotonic entailment relation including,
but not limited to, all the members of the description logic family of knowledge representation
formalisms.

Our approach extends on ideas originally introduced for static probabilistic reasoning. The essence
of the method is to divide an ontology into different contexts, which are identified by a consistent set
of propositional variables from a previously chosen finite set of variables $V$.
The probabilistic knowledge is expressed through a probability distribution over the valuations of $V$ 
which is encoded by a DBN. 

Interestingly, our formalism allows for reasoning methods that exploit the properties of both, the
ontological, and the probabilistic components. From the ontological point of view, we can use
suplemental reasoning to produce a context formula that encodes all the possible worlds from which
a wanted consequence can be derived. We can then use standard DBN methods to compute the 
probability of satisfying this formula.

This work represents first steps towards the development of a formalism combining well-known
ontology languages with time-evolving probabilities. First of all, we have introduced only the most
fundamental reasoning tasks. It is possible to think of many other problems like finding the most
plausible explanation for an observed event, or computing the expected time until a consequence is 
derived, among many others.

Finally, the current methods developed for handling DBNs, although effective, are not adequate
for our problems. To find out the probability of satisfying the context formula $\phi$, we need 
compute the probability of each of the valuations that satisfy $\phi$ at different points in time.
Even using methods that exploit the structure of the DBN directly, the information of the context 
formula is not considered. Additionally, with the most efficient methods to-date, it is
unclear how to handle the evidence over time effectively.
Dealing with these, and other related problems, will be the main focus of our future work.

\section{Acknowledgments}
\.Ismail \.Ilkan Ceylan is supported by DFG within the Research Training Group ``RoSI'' (GRK 1907).
Rafael Pe\~naloza was partially supported by DFG through the Cluster of Excellence `cfAED,' while he was still 
affiliated with TU Dresden and the Center for Advancing Electronics Dresden, Germany.

\bibliographystyle{named}
\bibliography{library}

\end{document}